\pdfoutput=1
\documentclass[11pt]{article}

\usepackage{EMNLP2023}

\usepackage{times}
\usepackage{latexsym}

\usepackage[T1]{fontenc}
\usepackage[utf8]{inputenc}

\usepackage[compact]{titlesec} 
\titlespacing{\section}{4pt}{4pt}{4pt} % this reduces space between (sub)sections to 0pt, for example

\usepackage{microtype}

\usepackage{inconsolata}

\usepackage{amsmath}
\usepackage{amssymb}
\usepackage{amsthm}
\usepackage{hyperref}
\usepackage{graphicx}
\usepackage{caption}
\usepackage{subcaption}
\usepackage{todonotes}
\usepackage{algorithm}
\usepackage[noend]{algpseudocode}

\title{Coordinated Replay Sample Selection for Continual Federated Learning}

\author{Jack H.~Good$^1$\Thanks{\enspace Work done while the author was an intern at Amazon},
Jimit Majmudar$^2$,
Christophe Dupuy$^2$,
Jixuan Wang$^2$,  \\
{\bf
Charith Peris$^2$,
Clement Chung$^2$,
Richard Zemel$^2$,
Rahul Gupta$^2$} \\
$^1$Carnegie Mellon University \\
$^2$Amazon Alexa AI \\
\texttt{jhgood@cs.cmu.edu} \\
\texttt{\{mjimit, dupuychr, wjixuan, perisc, chungcle, rzemel, gupra\}@amazon.com}}
\date{}

\newtheorem{theorem}{Theorem}

\begin{document}

\maketitle

\begin{abstract}
Continual Federated Learning (CFL) combines Federated Learning (FL), 
the decentralized learning of a central model on a number of client devices 
that may not communicate their data, 
and Continual Learning (CL), 
the learning of a model from a continual stream of data without keeping the entire history. 
In CL, the main challenge is \textit{forgetting} what was learned from past data. 
While replay-based algorithms that keep a small pool of past training data 
are effective to reduce forgetting, 
only simple replay sample selection strategies have been applied to CFL in prior work, and no previous work has explored coordination among clients for better sample selection.
To bridge this gap, we adapt a replay sample selection objective based on loss gradient diversity to CFL
and propose a new relaxation-based selection of samples
to optimize the objective.
Next, we propose a practical algorithm to coordinate gradient-based replay sample selection across clients without communicating private data. 
We benchmark our coordinated and uncoordinated replay sample selection algorithms 
against random sampling-based baselines
with language models trained on a large scale de-identified real-world text dataset. We show that gradient-based sample selection methods both boost performance and reduce forgetting compared to random sampling methods, with our coordination method showing gains early in the low replay size regime (when the budget for storing past data is small).
%, as well as several synthetic CFL data sets generated from public data using a new method designed to better capture the challenges of realistic applications of CFL.
\end{abstract}

\section{Introduction}
The ubiquity of personal devices with a network connection,
such as smart phones, watches, and home devices,
offer a rich source of data for learning problems such as
language modeling or facial recognition.
The conventional approach is to collect all the data into one location and use dedicated hardware to learn a model;
however, the privacy risk associated with communicating personal data makes this approach unsuitable for many applications.
\textit{Federated learning} (FL) offers a solution by learning a central model
via distributed training across user-owned devices, without communicating any data to the central server.

In addition, the devices may produce a continual stream of data and,
due to storage constraints and/or privacy restrictions,
be able to keep only a limited amount of data at a time.
Thus \textit{continual federated learning} (CFL) has recently emerged as a prominent topic in machine learning research.
CFL incorporates methods from \textit{continual learning} (CL),
where a model is periodically fine-tuned on new data.
The main challenge for CL is \textit{catastrophic forgetting},
a phenomenon where fine-tuning on new data causes a reduction of performance on past data.
This is harmful to long-term generalization, 
especially when different time periods comprise different tasks,
or when the data distribution shifts over time or presents seasonality.

Among various methods,
\textit{episodic replay}, wherein a small, fixed-size \textit{replay buffer} of past data
is kept and used for fine-tuning along with new data,
has proven to be among the most effective strategies to reduce forgetting and improve performance of the final model
in both CL \cite{verwimp21} and CFL \cite{guo21,Dupuy2023}.
However, only basic replay sample selection strategies, including random sampling and iCaRL \cite{rebuffi17},
have been applied to CFL \cite{guo21}.
To bridge this gap, we adopt the selection objective from gradient-based sample selection (GSS) \cite{aljundi19a}, 
a more recent approach that selects replay samples based on the diversity of their gradients. We propose a new relaxation-based selection method
that results in selections closer to optimal compared to methods from prior work.

Any replay sample selection method from CL can be used for CFL by applying it independently at each client.
However, CFL presents a yet-unexplored opportunity 
for the central server to coordinate the selection of replay samples across clients,
that is, choose samples such that the union of all clients' replay buffers, rather than each individual buffer, is optimal.
The main challenge is that, to ensure privacy, the data cannot be communicated to the server,
so selection techniques from CL can not be applied directly.
Building on our relaxation-based selection approach,
we propose the first server-coordinated replay sample selection approach for CFL.
By introducing auxiliary variables that make the objective of the relaxation 
separable across clients,
we enable an alternating minimization (more generally called block coordinate descent) process
whereby the optimization alternates between the server and the clients in parallel,
all while maintaining communication volume and privacy very similar to standard FL training.

Our novel contributions are
1) a relaxation-based approach 
to select replay samples that maximize loss gradient diversity;
2) a practical algorithm
for coordinated selection of replay samples
to maximize gradient diversity jointly across many clients
without sacrificing privacy
or substantially increasing communication or computation cost;
and 3) an empirical analysis of the effect of these strategies
on performance and forgetting
on a language modeling problem using real-world voice assistant data
with heterogeneity across clients and time periods.

\section{Related work}
\label{sec:related-work}

FedAvg \cite{mcmahan17} is a standard FL algorithm
wherein the server sends an initial model to a random sample of clients, 
each client in parallel fine-tunes the model with its local data and sends it back to the server,
and the server averages their weights to get a new central model.
This is repeated for a number of rounds.
If the clients are heterogeneous (have non-i.i.d. data distributions), 
then the weight averaging results in \textit{client drift}.
As a result, convergence rates of algorithms based on FedAvg generally get worse with client heterogeneity
\cite{wang19, karimireddy20, li20, reddi20}.
Several variations of FedAvg have been proposed to address challenges such as client drift
\cite{zhao18, wang19, li20, reddi20, karimireddy20}. % TODO should I go into detail on these?
%in particular, SCAFFOLD \cite{} uses variance reduction to eliminate dependence on client heterogeneity.
The replay sample selection strategies proposed in this paper are orthogonal to the particulars of the FL algorithm;
for our evaluation, we use standard FedAvg.
%Recent FL works commonly benchmark algorithms using FL data constructed from public data sets, such as MNIST and CIFAR, 
%with latent Dirichlet allocation (LDA) to simulate realistically diverse clients \cite{hsu19, yurochkin19}

``Continual learning'' can refer to several related problems,
but in this work, we consider the problem of learning a single task without forgetting
from a continual stream of data, usually by periodic fine-tuning,
with some limitations such as hardware capacity
precluding the retention of the full history of data.
The distribution of data may shift over time.
Common approaches to reduce forgetting are to
apply regularization penalizing the difference in weights between the current model previous models \cite{kirkpatrick17};
keep a small set of historical data
and project loss gradients such that they do not increase the loss on
these historical data \cite{lopez-paz17, chaudhry18, guo20};
or keep a small set of historical data to include during training
\cite{rebuffi17, aljundi19a, aljundi19b, borsos20}.
The last approach, called \textit{episodic replay} or \textit{rehearsal},
has been shown to be especially effective to reduce forgetting
in both CL \cite{verwimp21} and CFL \cite{guo21,Dupuy2023}.
In particular, gradient-based sample selection (GSS) \cite{aljundi19a} is an episodic replay strategy
that chooses replay samples to maximize the diversity of the loss gradients.
It is shown to outperform other strategies
and is the foundation for our proposed CFL methods.
%CL works also commonly benchmark using data constructed from public data sets,
%but but split into time periods instead of clients
%by applying strategies such as permuting images \cite{aljundi19a}
%or partitioning into label groups \cite{rebuffi17}.

Continual federated learning (CFL)
%, also sometimes called federated continual learning,
is a setting where each client receives a continual stream of data
and federated learning is periodically applied to update a central model.
This setting faces challenges of both heterogeneity across clients, as in FL,
and heterogeneity across time steps, as in CL.
CFL works that focus on improving performance by reducing forgetting,
like this one,
include the following:
\cite{yao20} applies model regularization methods from CL to FL, 
but focuses on improving generalization of FL by reducing client drift;
\cite{guo21} proposes a general CFL framework with convergence analysis
and applies CL techniques including model regularization, generative data augmentation,
and episodic replay strategies including naive random sampling and iCaRL \cite{rebuffi17}, 
finding that episodic replay outperforms the other CL strategies by a wide margin,
with the naive method being superior;
\cite{usmanova21} uses a distillation strategy with both central and past local models as teachers
for new local models;
\cite{jiang21} uses parameter masking to preserve and reuse knowledge;
and \cite{casado20} proposes a different take on CFL using lightweight models with ensemble methods,
focusing mainly on practical limitations of low-power devices,
but also discussing applicability to single-task CL problems with distribution shift.
To the best of our knowledge, we are the first to apply gradient-based replay sample selection methods to CFL
and the first to propose a server-coordinated approach.
Other CFL works focus on FL challenges such as client interference \cite{yoon21}
or variable sampling rate, device capabilities, latency, and availability issues \cite{chen20}.
%CFL data sets are often created from public data by combining data set construction techniques
%from CL and FL work,
%as the data must be split across both clients and time periods.

\section{Problem Formulation}
\label{sec:problem-formulation}

In FL, each client $m\in[M]$ has a set $X_m$ of samples of size $n_m=|X_m|$, and we aim to find a model $w$ that solves the optimization problem
\begin{equation}\label{eq:loss-FL}
\begin{aligned}
    \min_w \sum_{m\in[M]} \frac{n_m}{n}\ell(X_m; w)
\end{aligned}
\end{equation}
where $\ell$ indicates a client-level aggregate loss function and $n=\sum_m n_m$ is the total number of samples. 
In CFL, the samples are further split into $T$ consecutive time periods, so each client $m\in[M]$ and time period $t\in[T]$ has samples $X_{m,t}$ of size $n_{m,t}=|X_{m,t}|$, and we aim to find a model $w$ that minimizes
\begin{equation}
\label{eq:loss-CFL}
\begin{aligned}
    \min_w \sum_{m,t} \frac{n_{m,t}}{n}\ell(X_{m,t}; w)
\end{aligned}
\end{equation}
with $n=\sum_{m,t} n_{m,t}$ the total number of samples.
Since data is generated sequentially and that user-owned devices typically have limited storage, at time period $t$ each client only has access to the data generated during $t$ and a small subset of the past data. Thus in a CFL setting, we learn a series of models $w_1,\dots,w_T$, with the goal that $w_T$ minimizes (\ref{eq:loss-CFL});
each $w_t$, for $t \in [T]$, is trained on $X_{1,t},\dots,X_{M,t}$ using Federated Learning with initialization from $w_{t-1}$, except $w_1$, which is initialized randomly or pre-trained, e.g., on publicly available data.

\section{Episodic Replay Strategies}
\label{sec:strategies}

For each $t$, $w_t$ is trained on $X_{1,t},\dots,X_{M,t}$, 
so we may expect that $w_t$ minimizes $\sum_m \frac{n_{m,t}}{n_t}\ell(X_{m,t};w_t)$;
however, it is not necessarily true that $w_t$ minimizes $\sum_m \frac{n_{m,t'}}{n_{t'}}\ell(X_{m,{t'}};w_t)$ for $t'<t$
because training on later data can result in \textit{forgetting}.
\textit{Episodic replay} is a simple and effective remedy whereby, at each time period $t$,
each client $m$ has a \textit{replay buffer} $R_{m,t}$ containing at most $N_m$ data from $X_{m,1},\dots,X_{m,t-1}$,
where $N_m$ is the replay buffer size for client $m$.
Then $w_t$ is trained on $X_{1,t}\cup R_{1,t},\dots,X_{M,t}\cup R_{M,t}$ using federated learning.
%While it is not necessarily true that $\ell(X_{m,t}\cup R_{m,t};w)$ 
%equals or even approximates $\sum_{\tau\in[t]}\frac{n_{m,\tau}}{n_{m,\leq t}}\ell(X_{m,\tau};w)$,
The purpose of the replay buffer is to alleviate forgetting 
and ultimately result in a good performing model across time periods,
%(i.e., according to (\ref{eq:loss-CFL})),
and it has been shown in numerous 
CL \cite{rebuffi17, aljundi19a,  aljundi19b, borsos20, verwimp21}
and CFL \cite{guo21,Dupuy2023}
works that episodic replay is effective in accomplishing that.
The defining feature of an episodic replay strategy is how $R_{m,t+1}$ is selected from $X_{m,t}\cup R_{m,t}$.

We next describe several such sample selection strategies,
which we call \textit{uncoordinated} if the selection is made independently at each client,
or \textit{coordinated} if the selection is made jointly across clients.

\subsection{Random sample selection}

The most basic approach to replay sample selection is random sampling,
which is always uncoordinated.
We consider three baseline methods based on random sampling:
naive uniform, approximation of uniform, and fixed proportion proposed in \citep{Dupuy2023}, that we also describe in Appendix~\ref{sec:appendix-random-sampling}.

\subsection{Uncoordinated gradient-based selection}

Replay sample selection from CL can be adapted for uncoordinated sample selection in CFL
by applying them independently at each client.
Thus, to simplify notation for uncoordinated strategies, 
we can omit the client index $m$.
We adopt the strategy of \cite{aljundi19a} to select data into the replay buffer with high diversity of loss gradients,
that is, the gradient of the loss function with respect to model parameters,
as used to train the model.
% with some adjustments for our CFL setting.
At period $t$, we compute the loss gradients after training model $w_t$ on $X_t\cup R_t$.
For a given client at the end of period,
let $g_{i}\in\mathbb{R}^d$ be the loss gradient for sample $i\in[n'_{t}]$ for model $w_t$, 
with $d$ the number of model parameters,
and let $n'_t=|X_t\cup R_t|$ be the size of the data and replay buffer at time $t$.
As per \cite{aljundi19a}, we select the replay buffer $R_{t+1}$ 
to minimize the cosine similarity of gradients for selected samples.
\begin{equation}
\label{opt:gss-original}
\begin{aligned}
    \min_R \quad &
    \sum_{i,j\in R} \frac{\langle g_i, g_j\rangle}{\lVert g_i\rVert\lVert g_j\rVert} \\
    \text{s.t.} \quad & R\subseteq X_{t}\cup R_{t} \\
    & |R|=N
\end{aligned}
\end{equation}
This is generally NP-Hard to solve exactly \cite{aljundi19a}.
As a result, \cite{aljundi19a} proposes two methods to find approximate solution,
one using a greedy heuristic and the other using online clustering,
both of which are designed for efficiency in an online learning setting.
We propose a different approximation:
introduce variables $x_{i}$, $i\in [n'_t]$ 
and equivalently write Problem~(\ref{opt:gss-original}) as
\begin{equation}
\label{opt:gss-x}
\begin{aligned}
    \min_x \quad & x^T G^T G x\\
    \text{s.t.} \quad & x_i\in\{0,1\} \text{ for all } i \\
    & \sum_i x_{i}=N
\end{aligned}
\end{equation}
where $G\in\mathbb{R}^{d\times n'_t}$ is the matrix of gradient directions defined by $G_{:,i}=g_i/\lVert g_i\rVert$,
and let $R_{t+1}=\{i\mid x^*_{i}=1\}$ for solution $x^*$ to Problem~(\ref{opt:gss-x}).
We relax the domain of $x_i$ from $\{0,1\}$ to $[0,1]$.
The resulting problem is convex quadratic minimization and efficient to solve;
we finally let $R_{t+1}$ be the set of data with the top-$N$ values in the solution $x^*$.

Because the diagonal of $G^T G$ is 1, and because with high-dimensional gradients the off-diagonal elements of $G$ tend to be near 0,
$x^*$ tends to have values mostly close to the average $N/n'_t$,
so the solution resulting from the top-$N$ operation may be poor.
To alleviate this, we set the diagonal of $G^T G$ to zero,
which is equivalent to removing the $i=j$ terms in the sum of Problem~(\ref{opt:gss-original}),
which always sum to $N$, so this does not change the minimizer.
In the relaxation, however, it tends to result in $x^*$ values that are mostly 0 and 1,
reducing error from the top-$N$ selection,
but causing the relaxation to possibly be non-convex.
We find that both versions of our relaxation result in better solutions in practice
than the heuristics from \cite{aljundi19a} 
(see Figure~\ref{fig:obj5_50} in Section~\ref{sec:experiments}),
with the non-convex outperforming the convex relaxation.
Therefore we use the non-convex relaxation of Problem~(\ref{opt:gss-x})
%(with the diagonal of $G^T G$ set to 0)
for uncoordinated gradient-based replay sample selection.
This relaxation-based formulation also makes possible the coordinated selection strategy proposed in the next section.

Due to the high-dimension of the gradients, 
it is best in practice to compute $G^T G$ first and solve the relaxation of Problem~(\ref{opt:gss-x}) as written; 
however, the relaxed problem can also be expressed more intuitively as
\begin{equation}
\label{opt:gss-norm}
\begin{aligned}
    \min_x \quad & \lVert G x\rVert^2 \\
    \text{s.t.} \quad & x_i\in[0,1] \text{ for all } i \\
    & \sum_i x_{i}=N
\end{aligned}
\end{equation}
and interpreted as choosing the data with the minimal-magnitude sum of gradient directions for selected data.
This will help motivate the coordinated formulation proposed in the next section.

\subsection{Coordinated sample selection}
A coordinated sample selection strategy 
aims for the union of all clients' replay buffers, 
rather than each clients' individual buffer, to be optimal.
For example, in uncoordinated selection, many clients may choose similar samples for replay,
which results in suboptimal representation for training,
but coordinated selection aims for diversity across clients.
This means clients cannot independently make selections,
and because client data (hence gradients) may not be communicated to the server,
replay sample selection methods for CL cannot necessarily be adapted directly into coordinated CFL methods.

To make the gradient diversity objective of ($\ref{opt:gss-norm}$) coordinated,
we sum over data in the union of all clients' selected replay samples
instead of an individual client's.
\begin{equation}
\label{opt:coord-1}
\begin{aligned}
    \min_{x_1,\dots,x_M} \quad & \left\lVert \sum_{m} G_m x_m \right\rVert^2 \\
    \text{s.t.} \quad & x_{m,i}\in[0,1] \text{ for all } m,i \\
    & \sum_i x_{m,i}=N_m \text{ for all } m
\end{aligned}
\end{equation}
The obvious approach is to have each client $m$ send $G_m$ to the server and solve Problem~(\ref{opt:coord-1}) there;
however, not only can Problem~(\ref{opt:coord-1}) be resource-intensive to solve centrally with many clients,
but this also introduces a very large communication cost, 
as each column of $G_m$ is the size of the model itself.
More importantly, communicating gradients puts client data at 
risk since individual gradients are vulnerable to privacy attack \cite{zhu19}.
Therefore, the goal is to solve Problem~(\ref{opt:coord-1}) 
without substantial increase in communication or computation cost,
and without communicating data, gradients,
or anything else that reduces privacy.

We propose an alternating minimization process whereby an objective is minimized alternatively at the server
and in parallel at the clients.
Define auxiliary variables $h_1, \dots, h_M$ 
such that $h_m := G_m x_m - \frac{1}{M}\sum_{n \in [M]} G_n x_n$. Then we have
\begin{align*}
\left\lVert \sum\limits_{n\in [M]} G_n x_n \right\rVert^2 = M^2 \left\lVert G_m x_m - h_m \right\rVert^2
\end{align*}
for each $m \in [M]$. Adding over all $m \in [M]$,
%\begin{align*}
%\left\lVert \sum\limits_{n} G_n x_n \right\rVert^2 = M \sum_m \left\lVert G_m x_m - h_m \right\rVert^2,
%\end{align*}
%so Problem~(\ref{opt:coord-1}) can be equivalently written as 
Problem~(\ref{opt:coord-1}) can be equivalently written as 
\begin{equation}
\label{opt:coord-2}
\begin{aligned}
\min_{\substack{x_1,\dots, x_M \\ h_1,\dots, h_M}} \quad & M \sum\limits_{m \in [M]} \left\lVert G_m x_m - h_m \right\rVert^2 \\
\text{s.t.} \quad 
& x_{m,i}\in[0,1] \text{ for all } m,i \\
& \sum_i x_{m,i}=N_m \text{ for all } m \\
& h_m = G_m x_m - \frac{1}{M}\sum_n G_n x_n.
%\text{ for all } m.
\end{aligned}
\end{equation}
Next, relax Problem~(\ref{opt:coord-2}) to
\begin{equation}
\label{opt:coord-3}
\begin{aligned}
\min_{\substack{x_1,\dots, x_M \\ h_1,\dots, h_M}} \quad & M \sum\limits_{m} \left\lVert G_m x_m - h_m \right\rVert^2 \\
\text{s.t.} \quad 
& x_{m,i}\in[0,1] \text{ for all } m,i \\
& \sum_i x_{m,i}=N_m \text{ for all } m \\
& \sum\limits_{m} h_m = 0.
\end{aligned}
\end{equation}
Problem~(\ref{opt:coord-3}) is a relaxation of Problem~(\ref{opt:coord-2}) 
because the feasible set of the latter is a subset of the former. 
Theorem~\ref{thm:altmin}, proven in Appendix~\ref{sec:appendix-proof}, 
shows that this relaxation is tight.

\begin{theorem}
\label{thm:altmin}
If $x^*_1, \dots, x^*_M, h^*_1, \dots, h^*_M$ is an optimal solution of (\ref{opt:coord-3}), then it is also an optimal solution of (\ref{opt:coord-2}). 
\end{theorem}

As a consequence of Theorem~\ref{thm:altmin}, 
we can determine an optimal solution of the original coordinated problem (\ref{opt:coord-1}) 
by solving (\ref{opt:coord-3}). 
Moreover, if we fix $h$ and consider minimization only over $x$, 
then Problem~(\ref{opt:coord-3}) is separable over the $M$ clients.
This means we can use an alternating minimization 
(more generally called block coordinate descent \cite{wright15})
algorithm where 
each client $m$ optimizes w.r.t.~$x_m$ given $h_m$ in parallel and sends the resulting $G_m x^*_m$ to the server,
then the server optimizes w.r.t.~$h_1, \dots, h_M$ given $G_1 x_1,\cdots,G_M x_M$ and sends each resulting $h^*_m$ to client $m$.
%The process repeats for a number of iterations, or until convergence.
We initialize with $h_m=0$ for all $m$
so that the selection at zero iterations is the same as uncoordinated.
Pseudocode is given in Algorithm~\ref{alg:coord}.
It is shown by \cite{luo93} that block coordinate descent of a quadratic function
over a convex polyhedron converges at least linearly to a stationary point,
and in our case, that function is convex,
so this alternating process improves at every iteration
and converges at least linearly to an optimum of the coordinated objective on the relaxed domain.

Despite this, neither the data itself nor individual gradients need to be communicated.
What is communicated is targets $h_m$ and weighted sum loss gradients $G_m x_m$.
Each is just one gradient-sized vector 
rather than one per local data point as in sending the gradients themselves.
Thus the communication cost per iteration is the same as FedAVG.
The number of iterations can be chosen up-front as a hyperparameter
to trade off optimality of the selection with number of rounds
and total volume of communication,
or there could be a stopping condition
such as a threshold on change in loss indicating convergence.
As for privacy, FedAVG itself already makes a weighted combination of gradients public
when run with one batch per client;
it is simply the difference between the model parameters sent to the client
and the parameters the client sends back to the server.
In this sense, this algorithm is no less private than general FedAVG.

%We describe an intuitive interpretation of this process in Appendix~\ref{sec:appendix-intuition}
%and practical considerations in Appendix~\ref{sec:appendix-coordinated-practice}.

\algdef{SE}[SUBALG]{Indent}{EndIndent}{}{\algorithmicend\ }%
\algtext*{Indent}
\algtext*{EndIndent}

\begin{algorithm}
\caption{Coordinated replay sample selection.}
\label{alg:coord}
\begin{algorithmic}
\State at each client $m$: 
\Indent
    \State $G_m\gets$ gradients at $X_t\cup R_t$
    \State $h_m\gets 0$
\EndIndent
\Repeat 
    \State at each client $m$: 
    \Indent
        \State $\begin{aligned}
\textstyle x_m\gets \min_x \quad & \left\lVert G_m x - h_m \right\rVert^2 \\
 \textstyle\text{s.t.} \quad 
& x_{i}\in[0,1] \text{ for all } i \\
& \textstyle\sum_i x_i=N_m
\end{aligned}$ 
        \State send $G_m^Tx_m$ to the server
    \EndIndent
    \State at the server:
    \Indent
        \State $h_m\gets G_mx_m-\frac{1}{M}\sum_{n=1}^M G_nx_n$
        \State send each $h_m$ to client $m$
    \EndIndent
\Until convergence or max iterations
\State at each client $m$: 
\Indent 
    \State select $R_{t+1}$ from $X_t\cup R_t$
    by top-$N_m$ of $x_m$
\EndIndent
\end{algorithmic}
\end{algorithm}

To efficiently solve the minimization at clients
when gradients are large,
write
$\lVert G_m x - h_m \rVert^2
= x^T G_m^T G_m x + h_m^T h_m - 2h_m^T G_m x$
and pre-compute $G_m^T G_m$ and $h_m^T G_m$.
Also, $G_m$ is the same at each iteration of the alternating minimization,
so $G_m^T G_m$ may be computed just once.

\subsubsection{Intuitive interpretation}
\label{sec:appendix-intuition}

This alternating minimization process has an intuitive interpretation.
The goal is to choose replay data such that 
their sum of loss gradient directions across clients is close to zero.
The server sends a ``target sum gradient'' $h_m$ to each client $m$,
which is initially zero.
Each client independently chooses data so that its sum gradient $G_m x_m$
is as close as possible to its target $h_m$,
then sends the result $G_m x_m$ back to the server.
The server adjusts the targets $h_m$ 
to be as close as possible to the sum gradients actually returned by the clients,
while maintaining that $\sum_m h_m=0$.
In this sense, the back-and-forth process searches for the sum gradient assignments $h_m$
that sum to zero, and therefore targets the coordinated gradient diversity objective,
while being the most individually achievable by clients given their respective data.

%\subsection{CFL data creation}

\section{Experiments}
\label{sec:experiments}
We run experiments to demonstrate 
the quality of our relaxation-based sample selection
and the performance of models trained using CFL with the proposed sample selection strategies.
Additional experimental details and results are in Appendix~\ref{sec:appendix-experiments}.

\subsection{Near-optimality of relaxation-based selection}

\begin{figure}
\centering
\includegraphics[width=\linewidth]{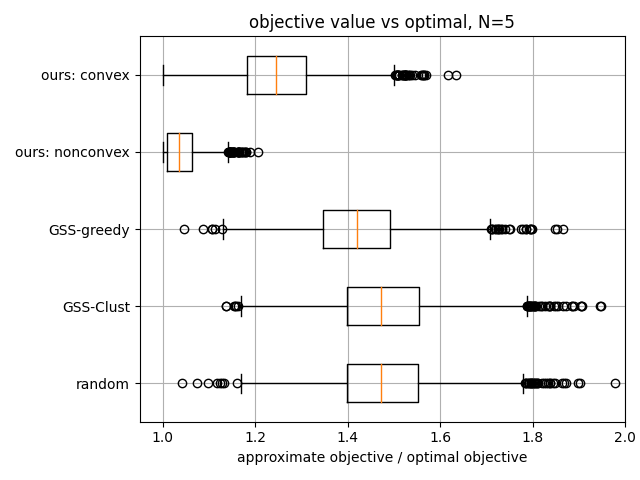}
\caption{Distribution of objective values vs.~optimal for approximate sample selection strategies.}
\label{fig:obj5_50}
\end{figure}

We empirically compare our relaxation-based sample selection
to the heuristic selection strategies proposed by \cite{aljundi19a},
as well as a random selection baseline.
We use randomly drawn vectors $g_i\in\mathbb{R}^{300}$
%, as described in Appendix~\ref{sec:appendix-relaxation},
and select $N=5$ out of $n=50$ data. We repeat the selection process 5000 times.
For each approach, we assess the quality of the selection
by comparing the resulting objective value
as in Problem~(\ref{opt:gss-original})
to the optimal value obtained by brute-force search (which is possible because $N$ and $n$ are small).

The distribution of objective ratios for each method
is shown in Figure~\ref{fig:obj5_50}.
Our relaxations achieve the best objective values,
with the non-convex relaxation being superior;
we expect this is because, with the convex relaxation, 
many $x^*$ values are close to the mean,
resulting in error during the top-$N$ operation that is not present
with the non-convex relaxation,
where $x^*$ values are close to 0 and 1.
In terms of objective value, 
the heuristic selection strategies from \cite{aljundi19a}
are only slightly better than random.

\subsection{Comparison of sample selection methods}

We compare CFL models learned using various replay buffer sizes and 
sample selection strategies,
including the proposed coordinated and uncoordinated strategies
as well as baseline strategies using random sampling.
We train a model with the TinyBERT architecture to a masked language modeling (MLM) task,
where the performance metric is perplexity (lower is better).
We choose TinyBERT \cite{jiao-etal-2020-tinybert} 
because distilled models with smaller footprints are more suitable for FL applications. 
We use 5 data sets, each of which comprises of automated transcriptions of utterances from a random sample of 1000 voice assistant users
split into 10 time periods of 5 weeks each:
the first 4 weeks are used for training and the remaining 1 week for testing.
Additional experiment details in Appendix~\ref{sec:appendix-main-experiments}.

\begin{figure}
\centering
\begin{subfigure}[t]{\linewidth}
\includegraphics[width=\linewidth]{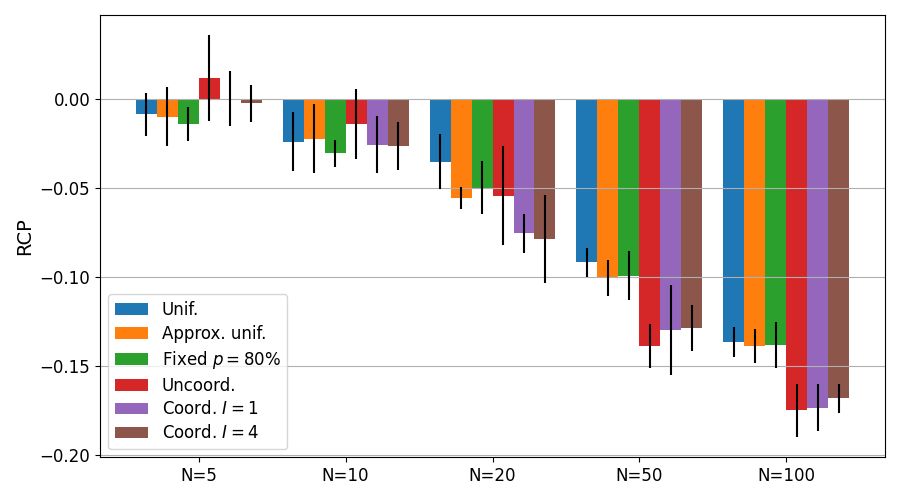}
\caption{All-period test set perplexity.}
\end{subfigure}
\begin{subfigure}[t]{\linewidth}
\includegraphics[width=\linewidth]{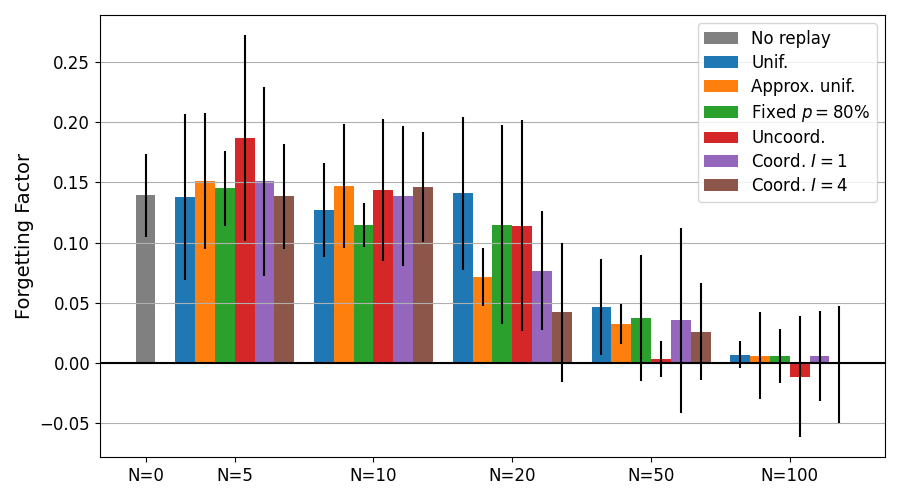}
\caption{All-period test set forgetting.}
\end{subfigure}
\caption{Relative change 
in perplexity (RCP) of models learned with various replay sample selection strategies.
Error bars show standard deviation over 5 different samples of clients.}
\label{fig:bar}
\end{figure}

All results are given in terms of relative change in perplexity (RCP),
that is, the relative change in perplexity for the experimental model
with respect to the model trained without episodic replay ($N=0$).
We also report the forgetting factor, 
defined as the difference in performance between the latest model 
%(i.e, after training on the latest time period data) 
and the best performance of the previous models on the same test set \cite{Dupuy2023}. 
A zero or negative value means that the latest model 
does not present forgetting on this test set;
a positive value means that a past model performs better than the latest model on this test set, 
which indicates forgetting.

Figure~\ref{fig:bar} shows the overall performance and forgetting factor
for each replay buffer size and sample selection strategy.
As expected, we see that the error and forgetting both decrease 
as the replay buffer size $N$ increases;
at $N=100$, there is close to no forgetting on average.
We also see that gradient-based sample selection 
increasingly outperforms random sample selection
as $N$ increases.
Coordinated sample selection appears to outperform uncoordinated sample selection with a low replay budget, $N\leq 20$.
There does not seem to be a notable difference between 1 and 4 iterations of coordinated optimization, suggesting that most of the benefit from coordinated selection is achieved after just one iteration.

\begin{figure}
\centering
\includegraphics[width=\linewidth]{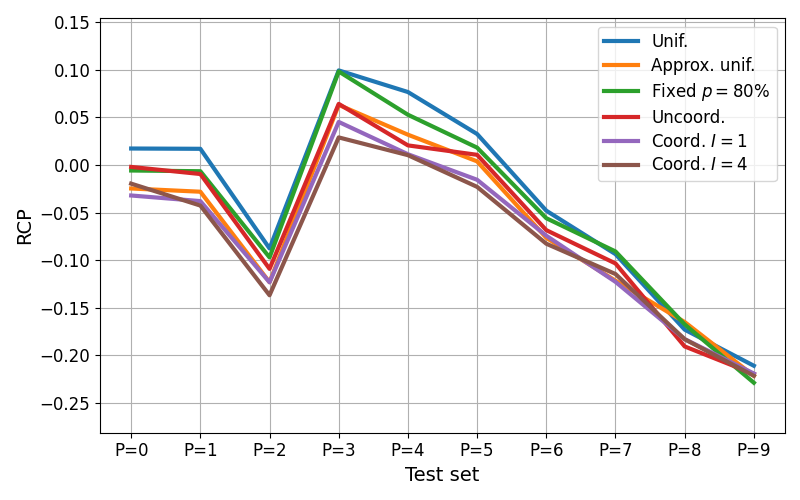}
\caption{Performance on each period for $N=20$.}
\label{fig:line}
\end{figure}

Figure~\ref{fig:line} shows the $N=20$ RCP results
for each period of the test set,
relative to the all-period test perplexity for the no-replay model.
As expected, with some exception, 
performance is generally better on more recent periods.
Also, the performance gap between methods is larger on earlier time periods,
with the coordinated methods consistently performing best on each time period
except the most recent ones.
Results for other $N$ are shown 
in Appendix~\ref{sec:appendix-main-experiments}.

\section{Discussion}
We proposed a new relaxation for gradient-based selection of replay samples in continual learning.
Based on this, we proposed the first algorithm 
for coordinated replay sample selection
in continual federated learning,
which converges to the optimal selection under our relaxation
while maintaining privacy and low communication cost.
Our experiments show that, compared to random sampling,
the gradient-based selection of replay samples
improves performance of the final model 
for various replay buffer sizes,
and coordinated selection improves for small buffer sizes.

\section{Limitations}
The reproducibility of this work is limited because the data used for some experiments is not public.
Moreover, training language models in a large CFL setting is extremely demanding of both time and computational resources.

\section*{Acknowledgements}
We thank Saleh Soltan for creating the BERT embeddings and encoder that were used in this work. 

\bibliographystyle{acl_natbib}
\bibliography{test}

\appendix

\section*{Appendix}

\section{Random sampling strategies}
\label{sec:appendix-random-sampling}
Here we describe the replay sample selection strategies
based on random sampling,
which were omitted from the main text to comply with page limits.

\textbf{Naive uniform:} each client samples $N$ data uniformly at random from $X_{t}\cup R_{t}$. 
This method is ``naive'' because the likelihood of selecting examples from the earliest periods decreases with time, 
which suggests higher vulnerability to catastrophic forgetting.

\textbf{Approximation of uniform:} each client samples $N n_{t}/n_{\leq t}$ data uniformly from $X_{t}$
and $N n_{<t}/n_{\leq t}$ data uniformly from $R_{t}$.
In this way, $R_{t+1}$ approximates a uniform sample from $X_{\leq t}$, the set of all data seen so far.
While this allows early time periods to continue to be represented, the representation of each individual period
reduces over time; after many time steps, the number of samples from even the most recent time period approaches 0.

\textbf{Fixed proportion $p\in(0, 1)$:} each client samples $p N$ data uniformly from $X_{t}$
and $(1-p) N$ data uniformly from $R_{t}$. 
Like naive uniform, the buffer contains fewer data from earlier periods, 
but the decrease is controlled by the chosen $p$ instead of customer activity. 

\section{Proof of Theorem~\ref{thm:altmin}}
\label{sec:appendix-proof}

\setcounter{theorem}{0}

\begin{theorem}
If $x^*_1, \dots, x^*_M, h^*_1, \dots, h^*_M$ is an optimal solution of (\ref{opt:coord-3}), then it is also an optimal solution of (\ref{opt:coord-2}). 
\end{theorem}

\begin{proof}
We first show that $x^*_1, \dots, x^*_M, h^*_1, \dots, h^*_M$ is a feasible solution of (\ref{opt:coord-2}). 
Since (\ref{opt:coord-3}) is convex, using the KKT optimality conditions, 
$x^*_1, \dots, x^*_M, h^*_1, \dots, h^*_M$ is optimal for (\ref{opt:coord-3}) 
if and only if it is feasible for (\ref{opt:coord-3}) 
and there exist non-negative vectors $u^*_1, \dots, u^*_M, v^*_1, \dots, v^*_M$, vector $w^*$, 
and scalars $\alpha^*_1, \dots, \alpha^*_M$ satisfying
\begin{itemize}
    \item $x^*_{m, i} u^*_{m, i}=0$ and $(1-x^*_{m, i}) u^*_{m, i}=0$ for all $m \in [M], i \in [N_m]$,
    \item $2 M G_m^T\left(G_m x^*_m - h^*_m \right) - u^*_m + v^*_m + \alpha^*_m e = 0 $ for all $m \in [M]$,
    \item $-2M \left(G_m x^*_m - h^*_m \right) + w^* = 0 $ for all $m \in [M]$.
\end{itemize}
Adding the last equation over all $m \in [M]$, we get $w^* = 2 \sum_m\left(G_m x^*_m - h^*_m \right)$ and therefore, 
for all $m \in [M]$,
\begin{align*}
    h^*_m &= G_m x^*_m - \dfrac{w^*}{2M} \\
    &= G_m x^*_m - \frac{1}{M} \sum_m\left(G_m x^*_m - h^*_m \right) \\
    &= G_m x^*_m - \frac{1}{M} \sum_m G_m x^*_m. \\ 
    &\left(\because \sum\limits_{m \in [M]}h^*_m=0 \right)
\end{align*}
This shows that $x^*_1, \dots, x^*_M, h^*_1, \dots, h^*_M$ is a feasible solution of (\ref{opt:coord-2}). 

Next we show the optimality of $x^*_1, \dots, x^*_M, h^*_1, \dots, h^*_M$ for (\ref{opt:coord-2}). 
Suppose for contradiction that $x'_1, \dots, x'_M, h'_1, \dots, h'_M$ is a feasible solution of (\ref{opt:coord-2}) 
such that
\begin{equation*}
    \sum_m \left\lVert G_m x'_m - h'_m \right\rVert^2 < \sum_m \left\lVert G_m x^*_m - h^*_m \right\rVert^2.
\end{equation*}
This contradicts the optimality of $x^*_1, \dots, x^*_M, h^*_1, \dots, h^*_M$ for (\ref{opt:coord-3}) 
since $x'_1, \dots, x'_M, h'_1, \dots, h'_M$ is a feasible solution of (\ref{opt:coord-2}).
\end{proof}

%\section{Additional details for coordinated sample selection}
%This section contains additional information 
%about the proposed coordinated replay sample selection algorithm.

\section{Experiment Details and Additional Results}
\label{sec:appendix-experiments}
This section contains additional details and results for the experiments.

\subsection{Near-optimality of relaxation-based selection}
\label{sec:appendix-relaxation}
For these experiments, 
the vectors $g_i\in\mathbb{R}^{d}$ with $d=300$, $i\in[n]$ 
were generated for each of $M=1000$ clients
by sampling from a random Gaussian mixture as follows.
Let the number of centers be $n_c=p+1$ with $p\sim\text{Poisson}(4)$,
then sample centers $c_{k,j}\sim\mathcal{N}(0,1)$ for $k\in[n_c]$, $j\in[300]$
and normalize such that each $\{c_{k,j}\mid k\in[n_c]\}$ has mean 0 and standard deviation 1.
Let $w\sim\text{Dir}(\mathbf{1}_{n_c})$, 
where $\mathbf{1}_{n_c}$ is the vector of length $n_c$ whose elements are 1,
then for each $i\in[n]$, 
sample $k\sim\text{Categorical}(w)$ 
and $g_{i,j}\sim\mathcal{N}(c_{k,j}, 1)$.
Finally, normalize the $g_i$ 
so that $\{g_{i,j}\mid i\in[n]\}$ has mean 0 and standard deviation 1.

\begin{figure}
\centering
\begin{subfigure}[t]{\linewidth}
\includegraphics[width=\linewidth]{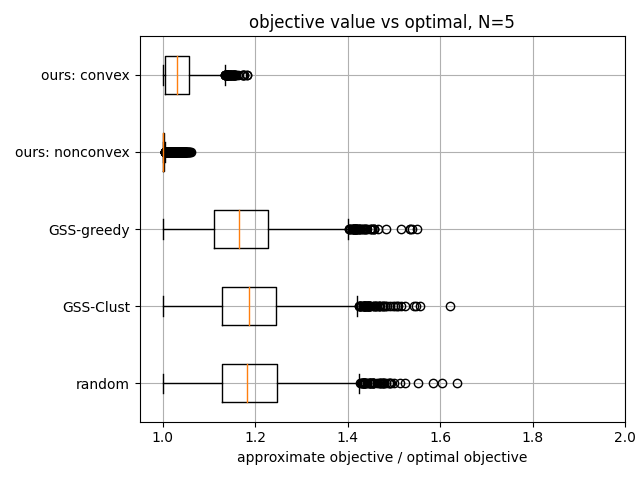}
\caption{$n=10$}
\end{subfigure}
%\begin{subfigure}[t]{\linewidth}
%\includegraphics[width=\linewidth]{results/obj5_20.png}
%\caption{$n=20$}
%\end{subfigure}
\begin{subfigure}[t]{\linewidth}
\includegraphics[width=\linewidth]{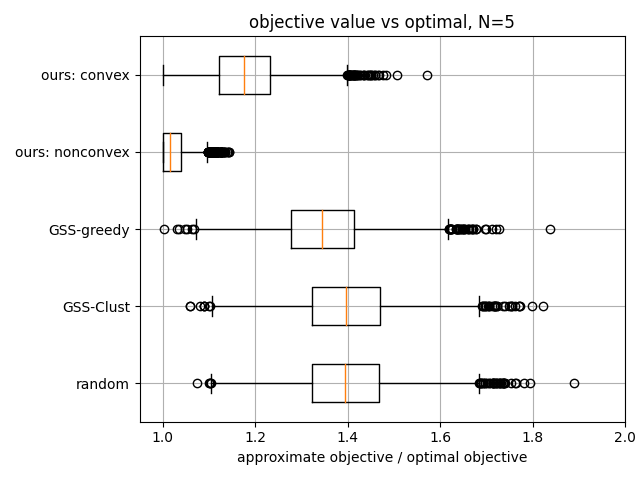}
\caption{$n=30$}
\end{subfigure}
%\begin{subfigure}[t]{\linewidth}
%\includegraphics[width=\linewidth]{results/obj5_40.png}
%\caption{$n=40$}
%\end{subfigure}
\begin{subfigure}[t]{\linewidth}
\includegraphics[width=\linewidth]{results/obj5_50.png}
\caption{$n=50$}
\end{subfigure}
\caption{Distribution of objective values vs.~optimal for approximate sample selection strategies.}
\label{fig:obj5}
\end{figure}

The results for $n=50$ were shown in the main text;
we show results for additional $n$ in Figure~\ref{fig:obj5}.
We see that the relative gap 
between the optimal and approximate selection increases with $n$
for all methods;
however, the relative difference between the approximate methods
is similar regardless of $n$.

\subsection{Comparison of sample selection methods}
\label{sec:appendix-main-experiments}

We use a TinyBERT model~\citep{jiao-etal-2020-tinybert} for our experiments, with L=4, H=312, A=12 and feed-forward/filter size=1200 where we denote the number of layers (i.e., Transformer blocks) as L, the hidden size as H, and the number of self-attention heads as A.

We ran 1555 parallelized experiments using p3.16x instances. Our training time per period per instance was approximately 21 minutes. Note that there was wide variance in training time values given that experiments for earlier periods take less time than experiments for later periods because of the replay buffer increasing training data size.

\begin{figure*}
\centering
\includegraphics[width=0.8\linewidth]{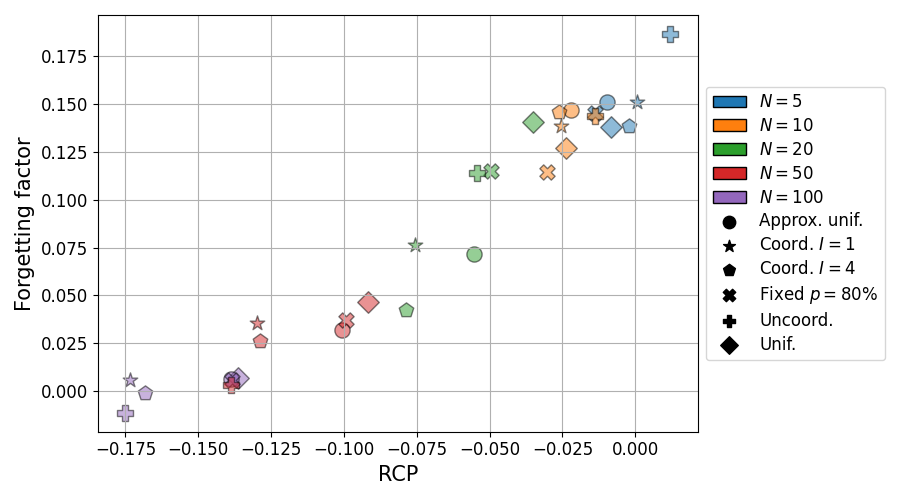}
\caption{Performance and forgetting results displayed as a scatterplot.}
\label{fig:scatterplot}
\end{figure*}

Figure~\ref{fig:scatterplot} shows the overall results in the form of a scatterplot.
This contains the same information as Figure~\ref{fig:bar}, but visualized differently.
There is a clear strong correlation between forgetting factor and performance;
this support the idea that the models with better replay improve performance by reducing forgetting.

Figure~\ref{fig:lines} shows the performance and forgetting broken down by the period of the test set,
as in Figure~\ref{fig:line},
but for both evaluation metrics and for all $N$.

\begin{figure*}
\centering
\begin{subfigure}[t]{0.4\linewidth}
\includegraphics[width=\linewidth]{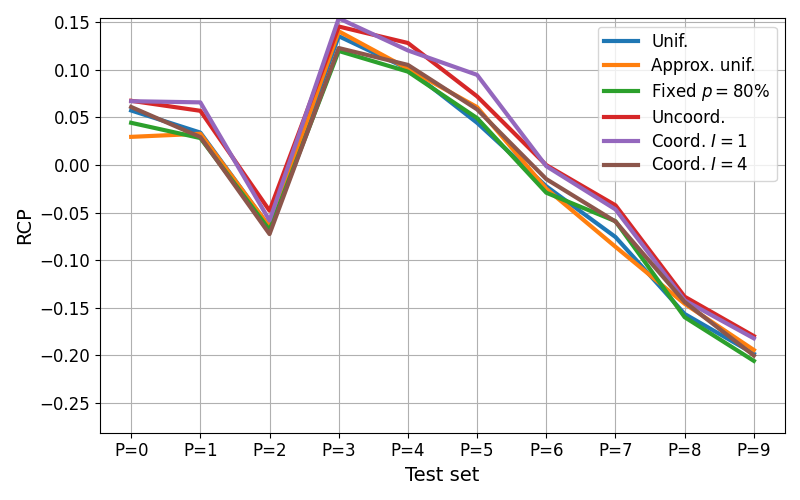}
\caption{$N=5$ performance.}
\end{subfigure}
\begin{subfigure}[t]{0.4\linewidth}
\includegraphics[width=\linewidth]{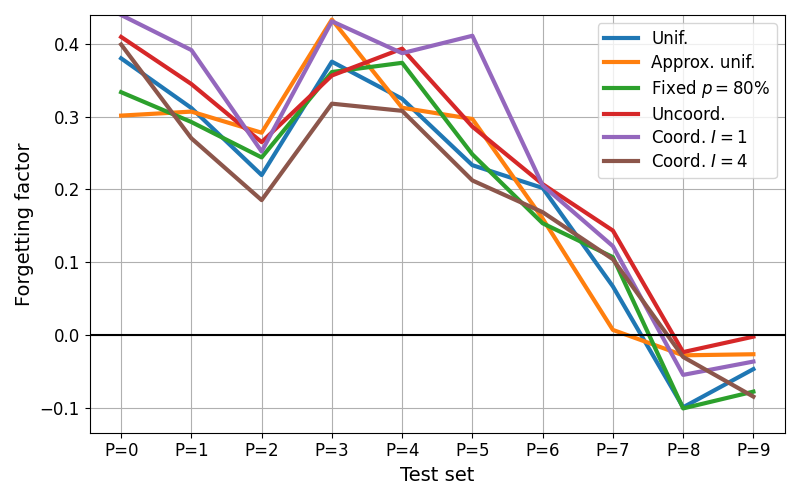}
\caption{$N=5$ forgetting.}
\end{subfigure}
\begin{subfigure}[t]{0.4\linewidth}
\includegraphics[width=\linewidth]{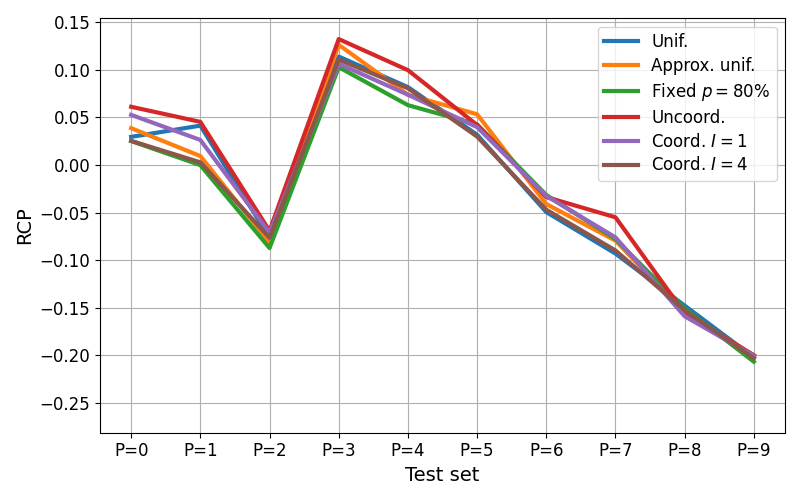}
\caption{$N=10$ performance.}
\end{subfigure}
\begin{subfigure}[t]{0.4\linewidth}
\includegraphics[width=\linewidth]{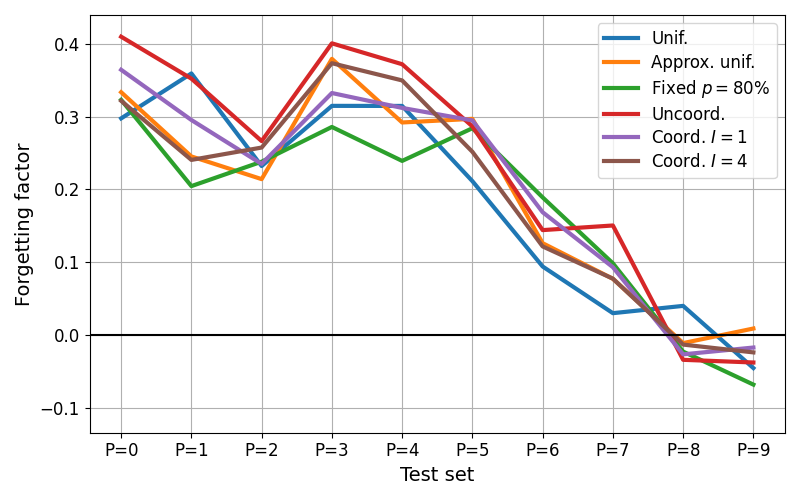}
\caption{$N=10$ forgetting.}
\end{subfigure}
\begin{subfigure}[t]{0.4\linewidth}
\includegraphics[width=\linewidth]{results/1k-relative-performance_line_plot_N20.png}
\caption{$N=20$ performance.}
\end{subfigure}
\begin{subfigure}[t]{0.4\linewidth}
\includegraphics[width=\linewidth]{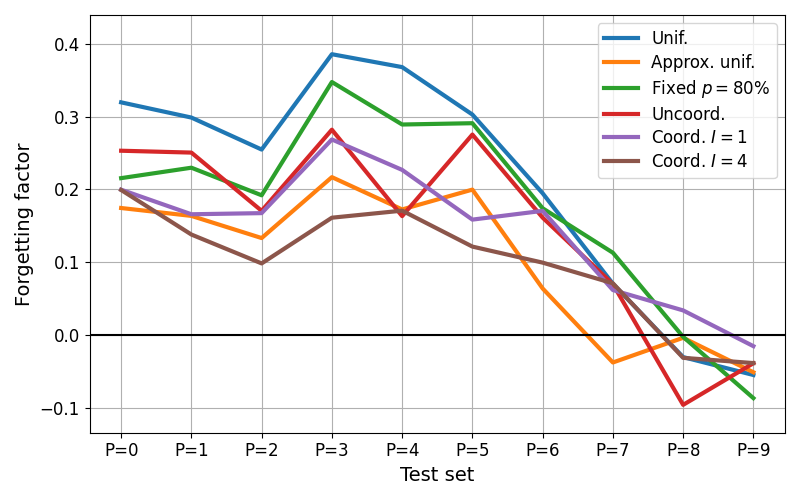}
\caption{$N=20$ forgetting.}
\end{subfigure}
\begin{subfigure}[t]{0.4\linewidth}
\includegraphics[width=\linewidth]{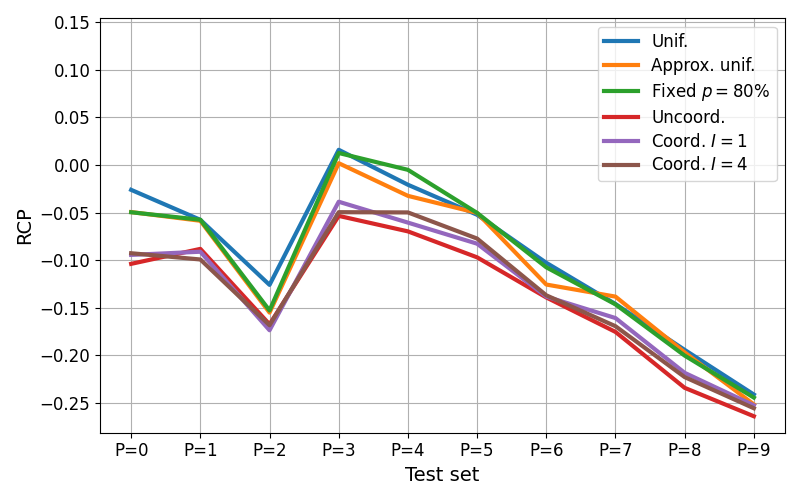}
\caption{$N=50$ performance.}
\end{subfigure}
\begin{subfigure}[t]{0.4\linewidth}
\includegraphics[width=\linewidth]{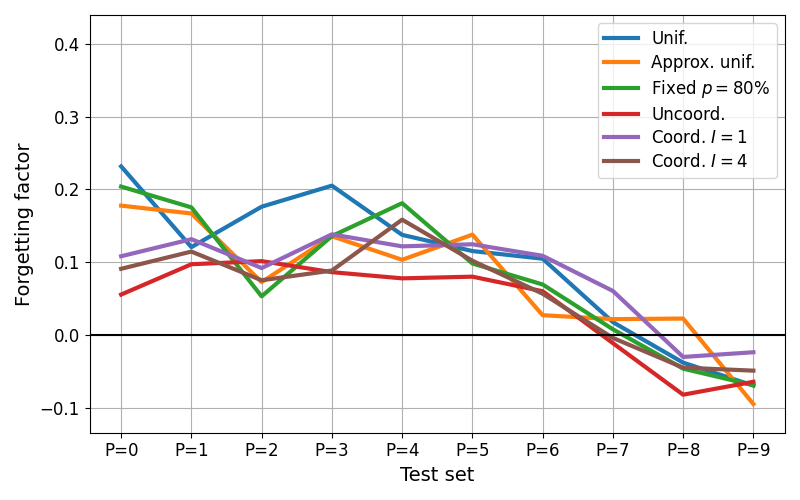}
\caption{$N=50$ forgetting.}
\end{subfigure}
\begin{subfigure}[t]{0.4\linewidth}
\includegraphics[width=\linewidth]{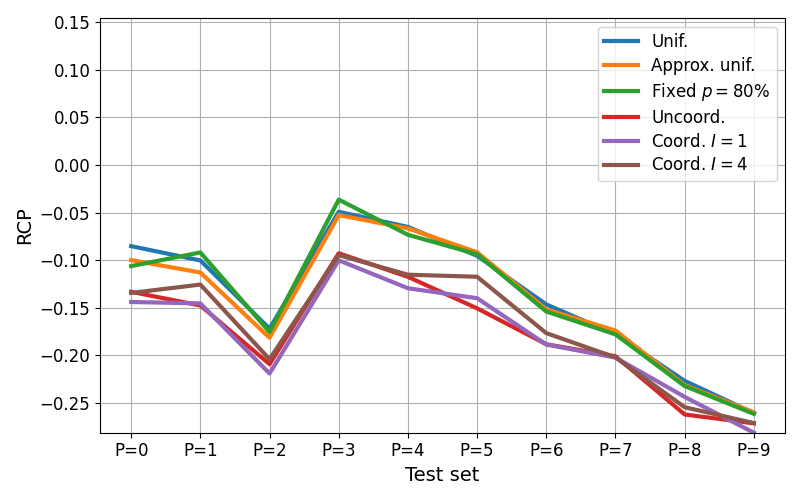}
\caption{$N=100$ performance.}
\end{subfigure}
\begin{subfigure}[t]{0.4\linewidth}
\includegraphics[width=\linewidth]{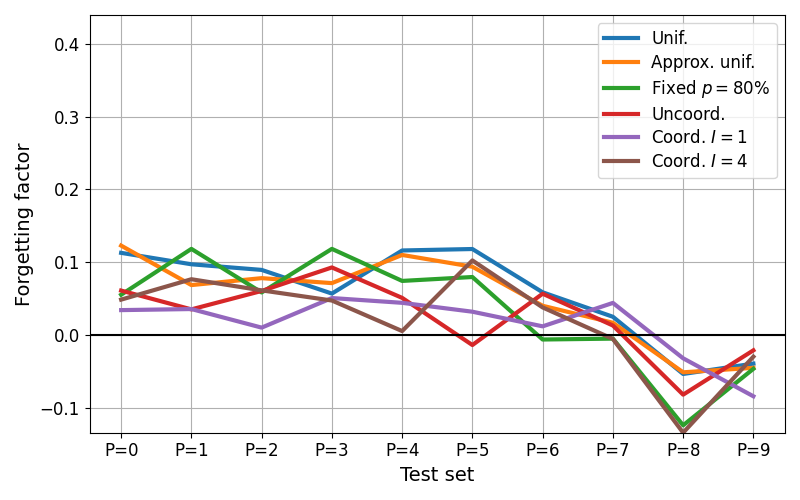}
\caption{$N=100$ forgetting.}
\end{subfigure}
\caption{Period test results for all $N$.}
\label{fig:lines}
\end{figure*}

\end{document}